\documentclass[sigconf]{acmart}
\usepackage{graphicx}
\usepackage{amsmath}
\usepackage{amsthm}

\usepackage{amssymb}
\usepackage{amsfonts}
\usepackage{bm}
\usepackage{caption}
\usepackage{multirow}
\usepackage[english]{babel}
\usepackage{adjustbox}
\usepackage{subfigure}
\usepackage{booktabs}
\usepackage{array}
\usepackage{algorithm}
\usepackage{algorithmic}
\theoremstyle{definition}
\newtheorem{defn}{Definition}

\newtheorem{corollary}{Corollary}

\usepackage{color}

\AtBeginDocument{%
  \providecommand\BibTeX{{%
    \normalfont B\kern-0.5em{\scshape i\kern-0.25em b}\kern-0.8em\TeX}}}

\copyrightyear{2021} 
\acmYear{2021} 
\setcopyright{acmcopyright}
\acmConference[KDD '21]{Proceedings of the 27th ACM SIGKDD Conference on Knowledge Discovery and Data Mining}{August 14--18, 2021}{Virtual Event, Singapore}
\acmBooktitle{Proceedings of the 27th ACM SIGKDD Conference on Knowledge Discovery and Data Mining (KDD '21), August 14--18, 2021, Virtual Event, Singapore}
\acmPrice{15.00}
\acmDOI{10.1145/3447548.3467430}
\acmISBN{978-1-4503-8332-5/21/08}

\begin{document}
\fancyhead{}

\title{Spatial-Temporal Graph ODE Networks for Traffic Flow Forecasting}

\author{Zheng Fang}
\authornote{These authors contributed equally to the work.}
\affiliation{%
  \institution{Key Laboratory of Machine Perception (Ministry of Education), Peking University}
}
 \email{fang\_z@pku.edu.cn}
 
\author{Qingqing Long}
 \authornotemark[1]
 \authornote{Work performed as a student of Peking University.}
\affiliation{%
  \institution{Alibaba Group}
}
 \email{lantu.lqq@alibaba-inc.com}

\author{Guojie Song}
\authornote{Corresponding Author.}
\affiliation{%
  \institution{Key Laboratory of Machine Perception (Ministry of Education), Peking University}
}
\email{gjsong@pku.edu.cn}

\author{Kunqing Xie}
\affiliation{%
  \institution{Key Laboratory of Machine Perception (Ministry of Education), Peking University}
}
\email{kunqing@cis.pku.edu.cn}

\begin{abstract}
Spatial-temporal forecasting has attracted tremendous attention in a wide range of applications, and traffic flow prediction is a canonical and typical example. The complex and long-range spatial-temporal correlations of traffic flow bring it to a most intractable challenge. Existing works typically utilize shallow graph convolution networks (GNNs) and temporal extracting modules to model spatial and temporal dependencies respectively. However, the representation ability of such models is limited due to: (1) shallow GNNs are incapable to capture long-range spatial correlations, (2) only spatial connections are  considered and a mass of semantic connections are ignored, which are of great importance for a comprehensive understanding of traffic networks. To this end, we propose Spatial-Temporal Graph Ordinary Differential Equation Networks (STGODE).
\footnote{Codes are available at https://github.com/square-coder/STGODE} .
Specifically, we capture spatial-temporal dynamics through a tensor-based ordinary differential equation (ODE), as a result, deeper networks can be constructed and spatial-temporal features are utilized synchronously. To understand the network more comprehensively, semantical adjacency matrix is considered in our model, and a well-design temporal dialated convolution structure is used to capture long term temporal dependencies. We evaluate our model on multiple real-world traffic datasets and superior 
performance is achieved over state-of-the-art baselines.
\end{abstract}


\maketitle

\section{Introduction}
Spatial-temporal forecasting has been widely studied in recent years. It has large scale applications in our daily life, such as traffic flow forecasting \cite{guo2019attention, du2017traffic}, climate forecasting \cite{jones2017machine, buizer2016making}, urban monitoring system analysis \cite{longo2017crowd} and so on. For this reason, accurate spatial-temporal forecasting plays a significant role in improving the service quality of these applications. In this paper, we study one of the most representative in spatial-temporal forecasting, traffic flow forecasting, which is an indispensable component in Intelligent Transportation System (ITS). Traffic flow forecasting attempts to predict the future traffic flow given historical traffic conditions and underlying road networks.


This task is challenging principally due to the complex and long-range spatial-temporal dependencies in traffic networks. As an intrinsic phenomenon of traffic, the travel distances of different people vary a lot \cite{plotz2017distribution}, which means that nearby and distant spatial dependencies largely exist at the same time. As Fig \ref{fig:district} shows, a node is not only connected to its geographical neighbors but also distant relevant nodes. Furthermore, traffic flow series exhibit diversified temporal pattern for their distinct behavior attributes as Fig \ref{fig:roads} shows. Moreover, when the spatial attributes and temporal patterns are united, the complex interactions in between leading to an intractable problem for traffic flow forecasting.

\begin{figure}[htbp]
\centering
\subfigure[]{
\includegraphics[width=0.48\linewidth]{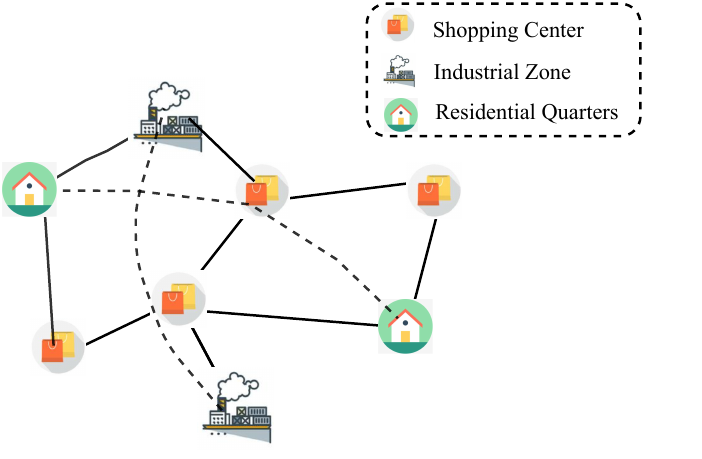}
\label{fig:district}}
\subfigure[]{
\includegraphics[width=0.43\linewidth]{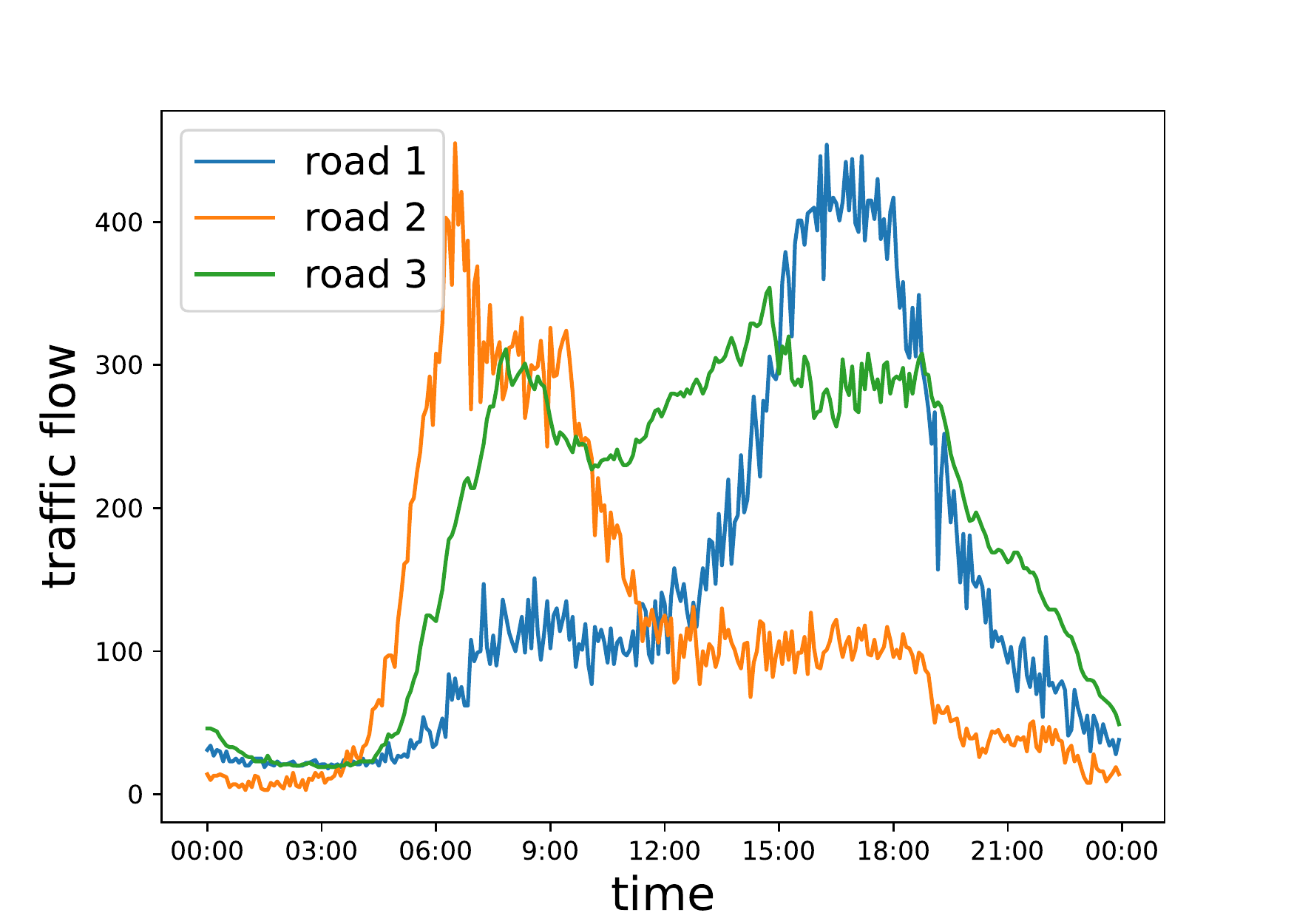}
\label{fig:roads}}
\caption{(a) shows the geographical and semantic connections of nodes. (b) shows examples of traffic flow with diverse patterns, like morning peak, evening peak and relatively steady patterns.}
\end{figure}

Graph Neural Networks (GNNs) for traffic forecasting have attracted tremendous attention in recent years. Owing to its strong ability to deal with graph-structured data, GNN enables to update node representations by aggregating representations from their neighbors, whereby GNN yields effective and efficient performance in various tasks like node classification and graph classification \cite{kipf2016semi, hamilton2017inductive, long2020graph, long2019hierarchical}. A large number of works have been proposed to utilize GNNs to extract spatial features in traffic networks, STGCN \cite{yu2018spatio} and DCRNN \cite{li2018diffusion} are the representative. Most of them combine GNNs with RNNs to obtain spatial representations and temporal representation respectively \cite{zhao2019t, pan2019urban}, and multiple works improve recurrent structure with convolution structure for better training stability and efficiency \cite{zhang2020spatio, fang2019gstnet}.

However, there are two problems that have been persistently neglected. On the one hand, most methods model spatial patterns and temporal patterns separately without considering their interactions, which restricts the representation ability of the models a lot. On the other hand, neural networks generally perform better with the stack of more layers, while GNNs benefit little from the depth. On the contrary, the best results are achieved when two-layer graph neural networks are cascaded, and more layers may lead to inferior performance in practice \cite{zhou2018graph, li2018deeper}. Ordinary GNNs have been proved to suffer from the over-smoothing problem, i.e. all node representations will converge to the same value with deeper layers. Such drawbacks severely limit the depth of GNNs and make it hardly possible to obtain deeper and richer spatial features. However, to the best of our knowledge, there are few works considering network depth in spatial-temporal forecasting, which is of great importance for capturing long-range dependencies.

In our Spatial-Temporal Graph Ordinary Differential Equation Network (STGODE), several components are elaborately designed to tackle the aforementioned problems. First, in order to depict spatial correlations from both geographical and semantic views, we construct two types of adjacency matrices, i.e. spatial adjacency matrix and semantic adjacency matrix, based on spatial connectivity and semantical similarity of traffic flow respectively. Second, motivated by residual networks \cite{he2016deep}, residual connections are added between layers to alleviate the over-smoothing problem.
Furthermore, it is proved that the discrete layers with residual connections can be viewed as a discretization of an Ordinary Differential Equation (ODE) \cite{chen2018neural}, and so a continuous graph neural network (CGNN) is derived \cite{xhonneux2020continuous}. Here in this paper, a continuous GNN with residual connections is introduced to avoid the over-smoothing problem and hence be able to model long-range spatial-temporal dependencies. Last but not least, a spatial-temporal tensor is constructed to consider spatial and temporal patterns simultaneously and model complex spatial-temporal interactions. We present the superiority of our model with a toy example. As Fig \ref{fig:schematic} shows, compared with STGCN, STGODE possesses a wider receptive field and thus can adjust outputs according to shifting circumstances to achieve better performance.
\begin{figure}[ht]
  \centering
  \includegraphics[width=0.75\linewidth]{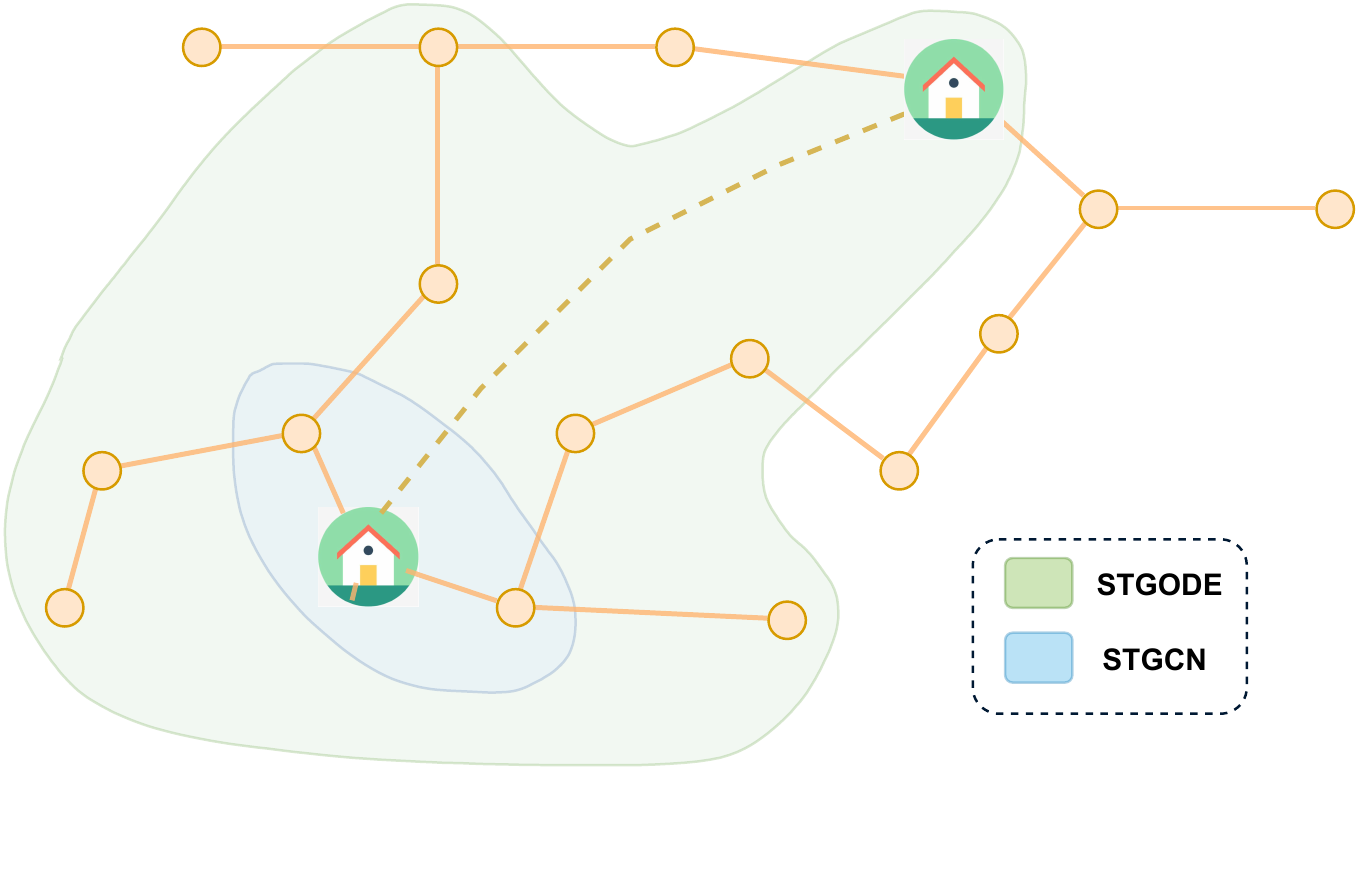}
  \caption{A performance schematic of STGODE}
  \label{fig:schematic}
\end{figure}

Our main contributions are summarized as follows,
\begin{itemize}
  \item We propose a novel continuous representation of GNNs in tensor form for traffic flow forecasting, which breaks through the limit of network depth and improves the capacity of extracting longer-range spatial-temporal correlations, and a theoretical analysis is given in detail.
  \item We utilize both spatial neighbors and semantical neighbors of road nodes to consider spatial correlations comprehensively.
  \item Extensive experiments are conducted on real-world traffic datasets, and the results show that our model outperforms existing baseline models.
\end{itemize}

\section{Related Work}
\subsection{Traffic Flow Forecasting}
In recent years a large body of research has been conducted on traffic flow forecasting, which has always been a critical problem in intelligent transportation systems(ITS)\cite{nagy2018survey}. Traffic flow forecasting can be viewed as a spatial-temporal forecasting task leveraging spatial-temporal data collected by various sensors to predict future traffic conditions. Classic methods, including autoregressive integrated moving average (ARIMA), k-nearest neighbors algorithm (kNN), and support vector machine (SVM), can only take temporal information into account, without considering spatial features.\cite{williams2003modeling,van2012short,jeong2013supervised}. Due to the limitation of modeling complex spatial-temporal relationships with classical methods, deep neural network models are proposed, which have been widely used in various challenging traffic prediction tasks. Specifically, FC-LSTM combines CNN and LSTM to model spatial and temporal relations through an extended fully-connected LSTM with embedded convolutional layers \cite{shi2015convolutional}. ST-ResNet utilizes a deep residual CNN network to predict citywide crowd flow \cite{zhang2017deep}, where the strong power of the residual network is exhibited. Despite impressive results that have been achieved, all above-mentioned methods are designed for grid data, thus not suitable for the traffic scene with graph-structured data.

 \subsection{Graph Neural Networks}
 GNN is an effective framework for the representation learning of graphs. GNNs follow a neighborhood aggregation scheme, where the computation of node representation is carried out by sampling and aggregating features of neighboring nodes \cite{kipf2016semi, hamilton2017inductive,long2021theoretically}. Strenuous efforts have been made to utilize graph convolution methods in traffic forecasting considering that traffic data is a classic kind of non-Euclidean structured graph data. For example, DCRNN \cite{li2018diffusion} view the traffic flow as a diffusion process and captures the spatial dependency with bidirectional random walks on a directed graph. STGCN \cite{yu2018spatio} builds a model with complete convolutional structures on both spatial and temporal view, which enables faster training speed with fewer parameters. ASTGCN \cite{guo2019attention} introduces attention mechanism to capture dynamics of spatial dependencies and temporal correlations. All these methods use two separate components to capture temporal and spatial dependencies respectively instead of simultaneously, thus STSGCN \cite{song2020spatial} makes attempts to incorporate spatial and temporal blocks altogether through an elaborately designed spatial-temporal synchronous modeling mechanism.

 Long-range spatial-temporal relationship, as a common-sense in traffic circumstances, is expected to be explored with deeper neural networks.
 However, the over-smoothing phenomenon of deep GNNs, which has been proved in a great number of studies \cite{zhou2018graph, li2018deeper}, will
 lead to similar node representations.
 Thus the depth of GNNs is restricted, and the long-range dependencies between nodes are largely ignored.

 \subsection{Continuous GNNs}
 Neural Ordinary Differential Equation(ODE) \cite{chen2018neural} models a continuous dynamic system based on parameterizing the derivative of the hidden state using a neural network, instead of specifying discrete sequences of hidden layers. CGNN \cite{xhonneux2020continuous} first extends this method to graph-structured data, which develops a continuous message-passing layer through defining the derivatives as combined representations of current and initial nodes. The key factor for alleviating the over-smoothing effect is the use of restart distribution, which motivates us in this paper. With proving simple GCN as a discretization of a kind of ODE, they characterize the continuous dynamics of node representations and enable deeper networks. To the best of our knowledge, there are no works about graph ODE in spatial-temporal forecasting.

\section{Preliminaries}
\begin{defn}
(Traffic network $\mathcal{G}$) We represent the road network as a graph $\mathcal{G}=(V,E,A)$, where $V$ is a set of $N$ nodes; E is a set of edges; $A\in \mathbb{R}^{N\times N}$ is an adjacency matrix. Here in this paper, two kinds of adjacency matrix are adopted, spatial adjacency matrix $A^{sp}$ and semantic adjacency matrix $A^{se}$.
\end{defn}
\begin{defn}
(Graph signal tensor $\mathcal{X}$) We use $\mathbf{x}_t^i\in \mathbb{R}^{F}$ to denote the observation of node $i$ at time $t$, and $F$ is the length of an observation vector. $X_t=\left(\mathbf{x}_t^1, \mathbf{x}_t^2, \cdots, \mathbf{x}_t^N\right)\in \mathbb{R}^{N\times F}$ denotes the observations of all nodes at time $t$. $\mathcal{X}=\left(X_1,X_2,\cdots,X_T\right) \in \mathbb{R}^{T \times N\times F}$ denotes the observations of all nodes at all time.
\end{defn}

\subsection{Problem Formulation}
Given the tensor $\mathcal{X}$ observed on a traffic network $\mathcal{G}$, the goal of traffic forecasting is to learn a mapping function $f$ from the historical $T$ observations to predict future $T'$ traffic observations,
\begin{equation*}
  \left[ X_{t-T+1}, X_{t-T+2}, \cdots, X_{t}; \mathcal{G}\right] \stackrel{f}{\longrightarrow} \left[ X_{t+1}, X_{t+2}, \cdots, X_{t+T'}\right].
\end{equation*}

\subsection{Regularized adjacency matrix}
Given an adjacency matrix $A\in \mathbb{R}^{N\times N}$, we typically normalize it as $\tilde{A}=D^{-\frac{1}{2}}AD^{-\frac{1}{2}}$, where $D$ is the degree matrix of $A$. $\tilde{A}$ has an eigenvalue decomposition  \cite{chung1997spectral} and the eigenvalues are in the interval
$[-1, 1]$. Negative eigenvalues can lead to unstable training process, thus a self-loop is commonly added to avoid it. The regularized form \cite{kipf2016semi} of $\tilde{A}$ is adopted in this paper:
\begin{equation}\label{2}
  \hat{A}=\frac{\alpha}{2}\left( I+D^{-\frac{1}{2}}AD^{-\frac{1}{2}}\right),
\end{equation}
where $\alpha \in (0,1)$ is a hyperparameter, as a result the eigenvalues of $\hat{A}$ are in the interval $[0, \alpha]$.

\subsection{Neural ODE}
We consider a continuous-time(depth) model,
\begin{equation}\label{3}
  \mathbf{x}(t) = \mathbf{x}(0) + \int_{0}^{t}\frac{\mathrm{d}\mathbf{x}}{\mathrm{d}\tau}\mathrm{d}\tau = \mathbf{x}(0) + \int_{0}^{t}f(\mathbf{x}(\tau),\tau)\mathrm{d}\tau,
\end{equation}
where $f(\mathbf{x}(\tau), \tau)$ will be parameterised by a neural network to model the hidden dynamic. We can backpropagate the process through an ODE solver without any internal operations \cite{chen2018neural}, which allows us to build it just as a block for the whole neural network.

\subsection{Tensor Calculation}
A tensor $\mathcal{T}$ can be viewed as a multidimensional array, and a tensor-matrix multiplication is defined on some mode fiber, for example,
\begin{equation}\label{eq:tensorm}
  \left(\mathcal{T}\times_2 M\right)_{ilk} = \sum_{j=1}^{n_2}\mathcal{T}_{ijk}\cdot M_{jl},
\end{equation}
where $\mathcal{T} \in \mathbb{R}^{n_1\times n_2\times n_3}, M\in \mathbb{R}^{n_2\times n'_2}, \mathcal{T}\times_2 M \in \mathbb{R}^{n_1\times n'_2\times n_3}$, $\times_2$ denotes that the tensor-matrix multiplication is conducted on mode-2, i.e. the second subscript. There are some mathematical properties about tensor-matrix multiplication which will be used in this paper,
\begin{itemize}
  \item $ \mathcal{T}\times_i M_1 \times_i M_2 = \mathcal{T}\times_i (M_1 M_2) $
  \item $ \mathcal{T}\times_i M_1 \times_j M_2 = \mathcal{T}\times_j M_2 \times_i M_1\ (i \neq j).$
\end{itemize}
Above properties can be easily proved with Eq \ref{eq:tensorm} through the multiplication rule.

\section{Model}
Figure \ref{fig:framework} shows the overall framework of our proposed model, i.e. Spatial-Temporal Graph ODE. It mainly consists of three components, two Spatial-Temporal Graph ODE (STGODE) layers composed of multiple STGODE blocks, a max-pooling layer, and an output layer. A STGODE block consists of two temporal dilation convolution (TCN) blocks and a tensor-based ODE solver in between, which is applied to capture complex and long-range spatial-temporal relationships simultaneously. The spatial adjacency matrix and the semantical adjacency matrix will be fed into the solver separately to obtain features from different levels. The details of the model will be described in the following section.

\begin{figure*}[htbp]
  \centering
  \subfigure[Framework]{
  \includegraphics[width=0.3\linewidth, ,trim=50 0 50 0,clip]{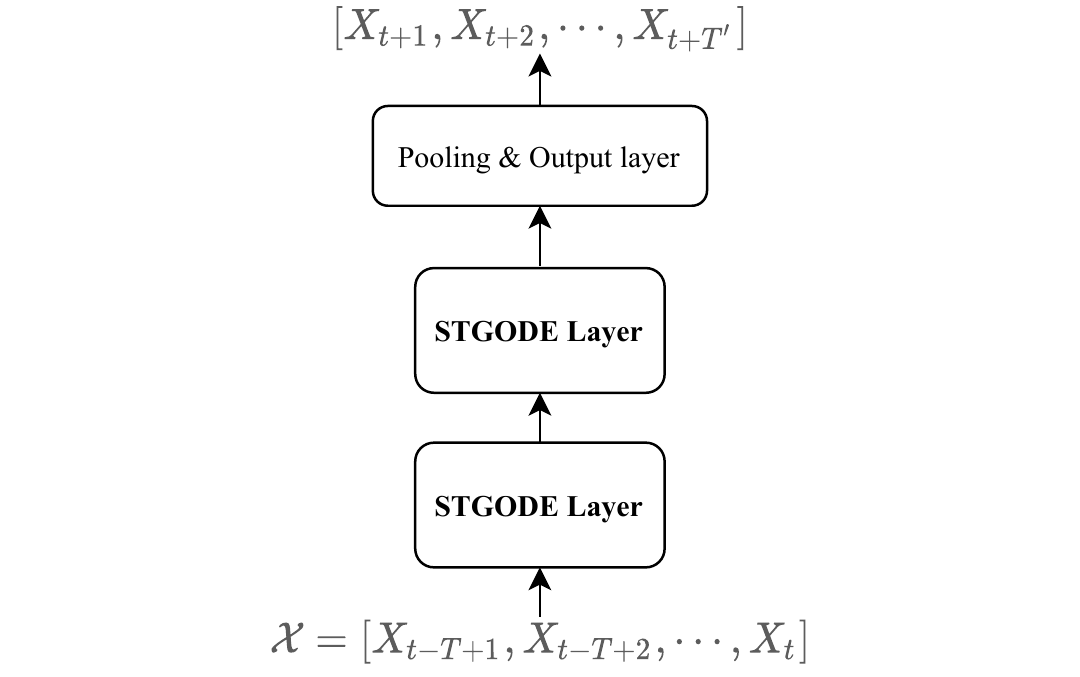}\label{fig:framework}}
  \subfigure[STGODE Layer]{
  \includegraphics[width=0.5\linewidth,clip]{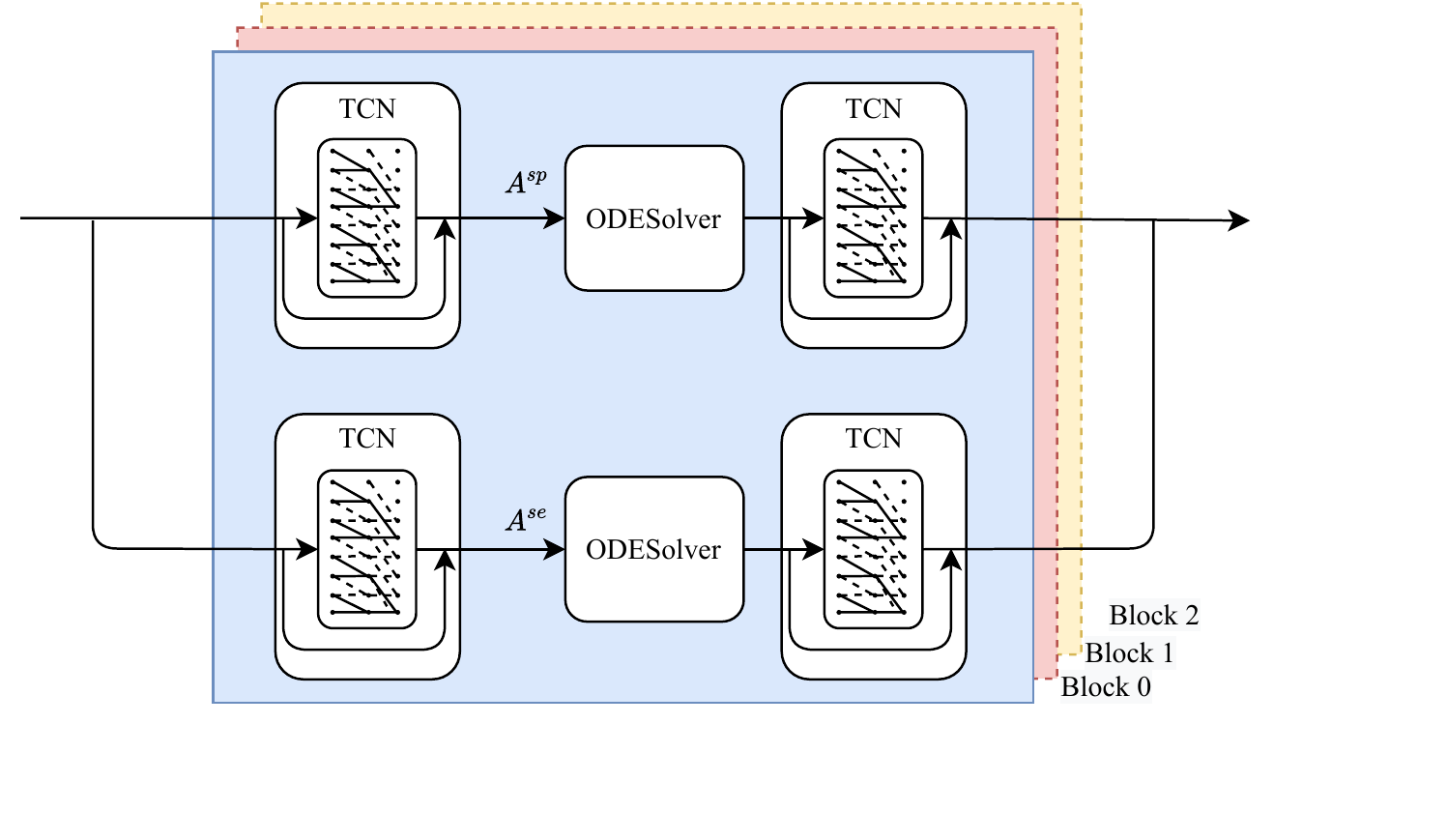}\label{fig:odelayer}}
  \caption{(a) is the framework of the STGODE network. Several STGODE blocks in parallel constitute a STGODE layer, and two STGODE layers are cascaded to extract higher-order features. (b) is the detail of STGODE blocks, where an ODE solver is sandwiched between two TCNs with residual connections and two kinds of adjacency matrices are utilized for more comprehensive characterization.}
\end{figure*}

\subsection{Adjacency Matrix Construction}
Two kinds of adjacency matrix are leveraged in our model. Following STGCN \cite{yu2018spatio}, the spatial adjacency matrix is defined as 
\begin{equation}\label{1}
  A^{sp}_{ij}=\left\{
  \begin{aligned}
    &\exp\left(-\frac{d_{ij}^2}{\sigma ^2}\right) &,\ &\text{if }\exp\left(-\frac{d_{ij}^2}{\sigma ^2}\right) \geq \epsilon \\
    &0 &,\ &\text{otherwise} \\
  \end{aligned}\right. ,
\end{equation}
where $d_{ij}$ is the distance between node $i$ and node $j$. $\sigma^2$ and $\epsilon$ are thresholds to control sparsity of matrix $A^{sp}$.

Besides, contextual similarities between nodes provide a wealth of information and should be taken into consideration. For example, similar traffic patterns are shared among roads near commercial areas regardless of the remote geographical distance, while such correlations cannot be revealed in spatial graph. To capture above mentioned semantic correlations, the Dynamic Time Warping (DTW) algorithm is applied to calculate the similarity of two time series \cite{berndt1994using}, which is superior to other metric methods on account of its sensitivity to shape similarity rather than point-wise similarity. As shown in Fig \ref{fig:dtw}, the point $a$ of series X will be related to the point $c$ but not $b$ of series Y by the DTW algorithm. Specifically, given two time series $X=(x_1, x_2, \cdots, x_m)$ and $Y=(y_1, y_2, \cdots, y_n)$, DTW is a dynamic programming algorithm defined as
\begin{equation}\label{4}
  D(i, j) = dist(x_i,y_j) + \min \left( D(i-1,j),D(i, j-1),D(i-1,j-1)\right),
\end{equation}
where $D(i,j)$ represents the shortest distance between subseries $X=(x_1, x_2, \cdots, x_i)$ and $Y=(y_1, y_2, \cdots, y_j)$, and $dist(x_i,y_j)$ is some distance metric like absolute distance. As a result, $DTW(X,Y)=D(m,n)$ is set as the final distance between $X$ and $Y$, which better reflects the similarity of two time series compared to the Euclidean distance.
\begin{figure}[hbtp]
  \centering
  \includegraphics[width=0.7\linewidth]{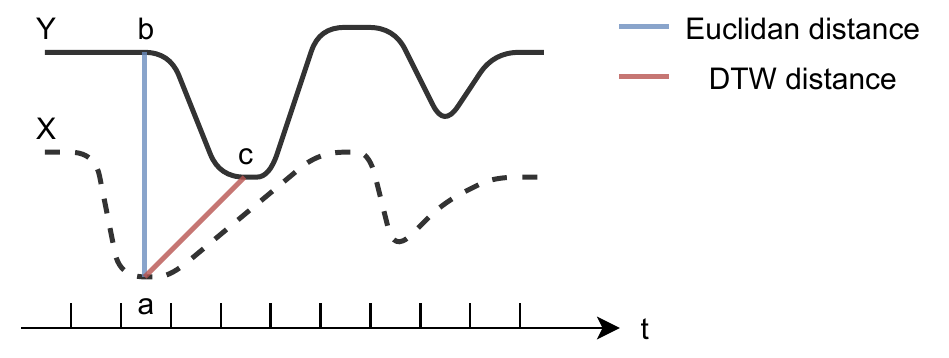}
  \caption{An example of the difference between the Euclidean distance and the DTW distance.}\label{fig:dtw}
\end{figure}

Accordingly, we define the semantic adjacency matrix through the DTW distance as following,
\begin{equation}\label{5}
  A^{SE}_{ij}=\left\{\begin{aligned}
  1, &\ DTW(X^i, X^j) < \epsilon \\
  0, &\ \text{otherwise}
  \end{aligned}
  \right.
\end{equation}
where $X^i$ denotes time series of node $i$, and $\epsilon$ determine the sparsity of the adjacency matrix.

\subsection{Tensor-based Spatial-Temporal Graph ODE}
GNNs update embeddings of nodes through aggregating features of their own and neighbors with a graph convolution operation. The classic form of convolution operation can be formulated as:
\begin{equation}\label{7}
  H_{l+1}=GCN(H_{l})=\sigma(\hat{A}H_{l}W),
\end{equation}
where $H_{l} \in \mathbb{R}^{N\times C}$ denotes the input of the $l$-th graph convolutional layer, $\hat{A} \in \mathbb{R}^{N\times N}$ is the normalized adjacency matrix, and $W\in \mathbb{R}^{C\times C'}$ is a learnable parameter matrix, which models the interaction among different features. However, such GNNs have been proved to suffer from over-smoothing problem when networks go deeper \cite{zhou2018graph, li2018deeper}, which largely limits the capacity of modeling long-range dependencies. For this reason, our STGODE block is proposed.

A discrete version is first shown as:
\begin{equation}\label{eq:discrete}
  \mathcal{H}_{l+1}=\mathcal{H}_l \times_1 \hat{A} \times_2 U \times_3 W + \mathcal{H}_0,
\end{equation}
where $\mathcal{H}_l \in \mathbb{R}^{N\times T \times F}$ is a spatial-temporal tensor representing nodes' hidden embedding of the $l$-th layer, $\times_i$ denotes the tensor-matrix multiplication on mode $i$, $\hat{A}$ is the regularized adjacency matrix, $U$ is the temporal transform matrix, $W$ is the feature transform matrix, and $\mathcal{H}_0$ denotes the initial input of GNN, which can be acquired through another neural network. Different from existing works, we treat the spatial-temporal tensor as input and hence enable to handle spatial information and temporal information simultaneously. The intricate spatial-temporal correlation is coupled through tensor multiplication on each mode. Motivated by CGNN \cite{xhonneux2020continuous}, a restart distribution $\mathcal{H}_0$ is involved to alleviate the over-smoothing problem.

Specifically, the expansion of Eq \ref{eq:discrete} is shown as below,
\begin{equation}\label{eq:expan}
  \mathcal{H}_{l}=\sum_{i=0}^{l}\left(\mathcal{H}_0 \times_1 \hat{A}^i \times_2 U^i \times_3 W^i \right),
\end{equation}
where we can see clearly that the output representation $\mathcal{H}_l$ aggregates information from all layers, that's to say, the final outputs collect information from all no more than $l$-order neighbors without losing initial features. To show the necessity of the restart distribution, let's suppose another version without $\mathcal{H}_0$ like $$\mathcal{H}_{l+1}=\mathcal{H}_l \times_1 \hat{A} \times_2 U \times_3 W,$$
where the final output will be
$$\mathcal{H}_{n}=\mathcal{H}_0 \times_1 \hat{A}^n \times_2 U^n \times_3 W^n.$$
Take $\hat{A}$ as a simple example, assuming $\hat{A}$ has an eigenvalue decomposition as $\hat{A}=P\Lambda P^T$, where $\Lambda = \text{diag}(\lambda_1, \lambda_2, \cdots, \lambda_m)$ is a diagonal matrix. Obviously,
\begin{equation}\label{0}
 \begin{aligned}
  \hat{A}^n &= P\text{diag}\left(\lambda_1^n, \lambda_2^n, \cdots, \lambda_m^n\right)P^T \\
  &= \lambda_1^n P\text{diag}\left(1, (\frac{\lambda_2}{\lambda_1})^n, \cdots, (\frac{\lambda_m}{\lambda_1})^n\right)P^T \\
  &\longrightarrow \lambda_1^n P\text{diag}\left(1, 0, \cdots, 0\right)P^T
\end{aligned}
\end{equation}
when n goes to infinity with $\lambda_1 > \lambda_2 > \cdots \lambda_m$. The diagonal elements converge to zero except the largest one, which causes much loss of information.

Such residual structure as Eq \ref{eq:discrete} is powerful but tough to train due to its enormous amount of parameters, thus we aim to extend the discrete formulation to a continuous expression. Intuitively, we replace $n$ with a continuous variable $t$, and view the expansion equation as a Riemann sum from $0$ to $n$ on $i$, which is,
\begin{equation}\label{11}
  \begin{aligned}
  \mathcal{H}_{n} &= \sum_{i=0}^{n}\left(\mathcal{H}_0 \times_1 \hat{A}^i \times_2 U^i \times_3 W^i \right) \\
  &= \sum_{i=1}^{n+1}\left(\mathcal{H}_0 \times_1 \hat{A}^{(i-1)\times \Delta t} \times_2 U^{(i-1)\times \Delta t} \times_3 W^{(i-1)\times \Delta t} \Delta t \right)
  \end{aligned}
\end{equation}
where $\Delta t = \frac{t+1}{n+1}$ with $t = n$. When n goes to $\infty$, we can formulate the following integral:
\begin{equation}\label{12}
  \mathcal{H}(t) = \int_{0}^{t+1}\mathcal{H}_0\times_1 \hat{A}^{\tau}\times_2 U^{\tau}\times_3 W^{\tau}\mathrm{d}\tau,
\end{equation}

The critical point here is to transform the residual structure to an ODE structure, obviously we already have an ordinary differential equation given by
\begin{equation}\label{eq:der1}
  \frac{\mathrm{d}\mathcal{H}(t)}{\mathrm{d}t}=\mathcal{H}_0\times_1 \hat{A}^{t+1}\times_2 U^{t+1}\times_3 W^{t+1},
\end{equation}
but $\hat{A}^{t+1}, U^{t+1}, W^{t+1}$ are intractable to compute especially when t is a non-integer. Motivated by the work in \cite{xhonneux2020continuous}, we have the following corollary.
\begin{corollary}
  The discrete update in Eq \ref{eq:discrete} is a discretization of following ODE:
\begin{equation}\label{eq:ode}
  \frac{\mathrm{d}\mathcal{H}(t)}{\mathrm{d}t} = \mathcal{H}(t)\times_1 \ln \hat{A} + \mathcal{H}(t)\times_2 \ln U + \mathcal{H}(t)\times_3 \ln W + \mathcal{H}_0,
\end{equation}
where $\mathcal{H}_0=f(\mathcal{X})$ is the output of upstream networks.
\end{corollary}

\begin{proof}
  Starting from Eq \ref{eq:der1}, we consider the second-derivative of $\mathcal{H}(t)$ through derivative rules,
  \begin{align}
    \frac{\mathrm{d}^2\mathcal{H}(t)}{\mathrm{d}t^2}
    =& \frac{\mathrm{d}\mathcal{H}(t)}{\mathrm{d}t}\times_1 \ln \hat{A} + \frac{\mathrm{d}\mathcal{H}(t)}{\mathrm{d}t}\times_2 \ln U + \frac{\mathrm{d}\mathcal{H}(t)}{\mathrm{d}t}\times_3 \ln W
  \end{align}
  Then integrate over $t$ in both sides of the above equation, we can get:
  \begin{equation}\label{eq:der2}
    \frac{\mathrm{d}\mathcal{H}(t)}{\mathrm{d}t} = \mathcal{H}(t)\times_1 \ln \hat{A} + \mathcal{H}(t)\times_2 \ln U + \mathcal{H}(t)\times_3 \ln W + const
  \end{equation}
  To solve the $const$, we can put Eq \ref{eq:der1} and Eq \ref{eq:der2} together, thus we have:
  \begin{align}
  const =& \mathcal{H}_0\times_1 \hat{A}^{t+1}\times_2 U^{t+1}\times_3 W^{t+1} \notag \\
  &- \left( \mathcal{H}(t)\times_1 \ln \hat{A} + \mathcal{H}(t)\times_2 \ln U + \mathcal{H}(t)\times_3 \ln W\right).
  \end{align}
  By letting $t\rightarrow -1$ mathematically, we can easily get $const = \mathcal{H}_0$. So the corollary is proved.
\end{proof}
In practice, we can approximate the logarithm operation with its first order of Taylor expansion, i.e. $\ln M \approx M - I$. As a result, a simpler form is obtained,
\begin{equation}\label{eq:simpleode}
  \frac{\mathrm{d}\mathcal{H}(t)}{\mathrm{d}t} = \mathcal{H}(t)\times_1 (\hat{A}-I) + \mathcal{H}(t)\times_2 (U-I) + \mathcal{H}(t)\times_3 (W-I) + \mathcal{H}_0
\end{equation}

\begin{figure}[htbp]
  \centering
  \subfigure[ODE Solver]{
  \includegraphics[width=0.65\linewidth]{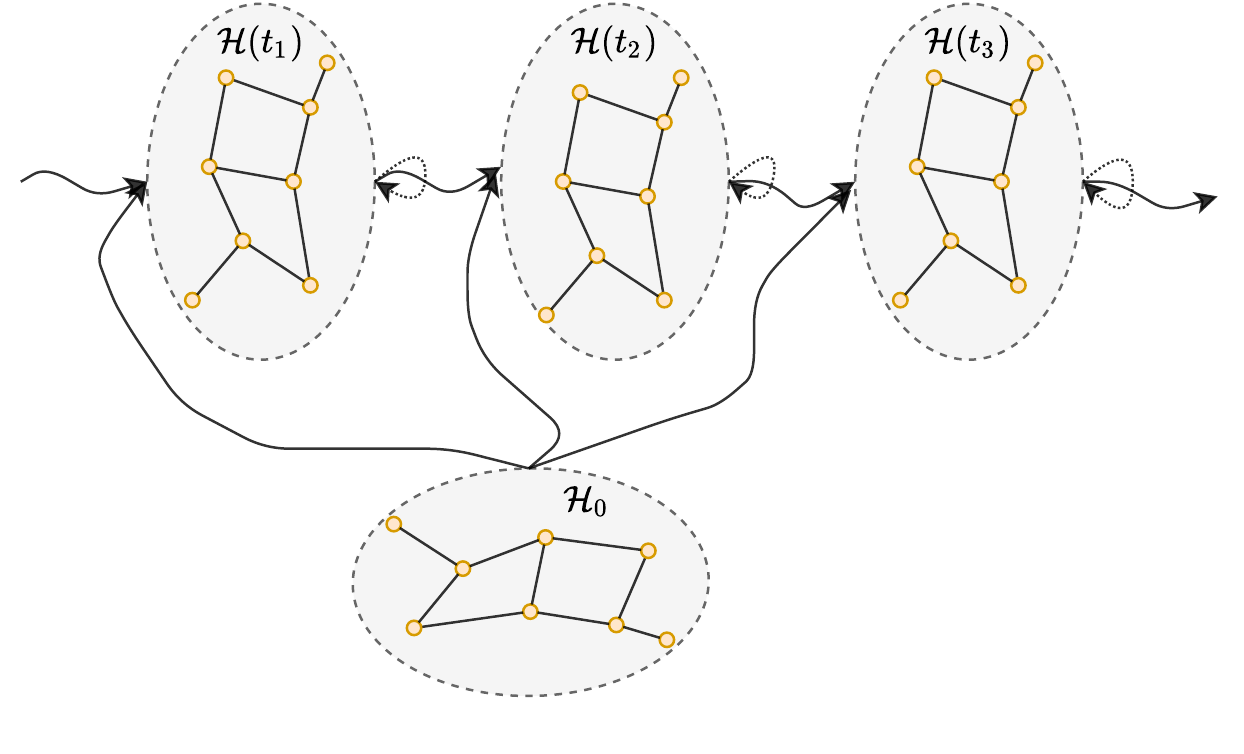}\label{fig:odesolver}}
  \subfigure[TCN]{
  \includegraphics[width=0.3\linewidth]{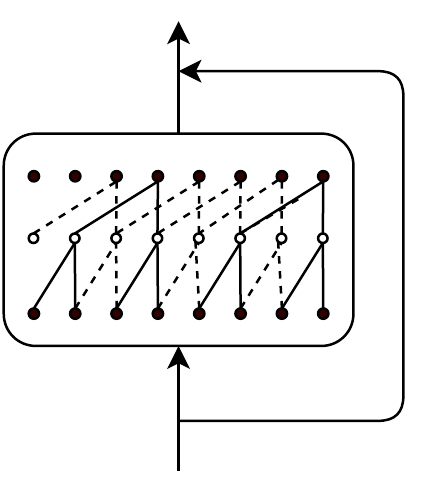}\label{fig:tcn}}
  \caption{(a) is the illustration of a ODE solver, which shows that the derivation of the hidden states is a function of current states and initial states. (b) represents the structure of TCN, which consists of a dilated convolution and a residual connection.}
\end{figure}

The above ODE we deduced can be analytically solved as the following corollary.
\begin{corollary}
  The analytic solution of the Eq \ref{eq:simpleode} is given by
  \begin{equation}\label{eq:intergration}
  \begin{aligned}
    \mathcal{H}(t) =& \mathcal{H}_0 \times_1 e^{(\hat{A}-I)t}\times_2 e^{(U-I)t} \times_3 e^{(W-I)t} \\
    &+ \int_{0}^{t}\mathcal{H}_0\times_1 e^{(\hat{A}-I)(t-s)}\times_2 e^{(U-I)(t-s)} \times_3 e^{(W-I)(t-s)}\mathrm{d}s
  \end{aligned}
  \end{equation}
\end{corollary}
\begin{proof}
  Suppose
  \begin{equation}
    \mathcal{H}^*(t) = \mathcal{H}(t)\times_1 e^{(\hat{A}-I)t}\times_2 e^{(U-I)t} \times_3 e^{(W-I)t},
  \end{equation}
  then we have
  \begin{align} \label{eq:hstar}
    \frac{\mathrm{d}\mathcal{H}^*(t)}{\mathrm{d}t}
     =& \mathcal{H}_0 \times_1 e^{(\hat{A}-I)t}\times_2 e^{(U-I)t} \times_3 e^{(W-I)t},
  \end{align}
  and this is derived from  Eq \ref{eq:simpleode}. Integrate Eq \ref{eq:hstar} on both sides, and we can get the following result,
  \begin{equation}\label{32}
    \mathcal{H}^*(t) = \mathcal{H}_0^* + \int_{0}^{t}\mathcal{H}_0\times_1 e^{(\hat{A}-I)\tau}\times_2 e^{(U-I)\tau} \times_3 e^{(W-I)\tau}\mathrm{d}\tau,
  \end{equation}
  and hence $\mathcal{H}(t)$ can be formulated as
  \begin{equation}\label{inter}
  \begin{aligned}
    \mathcal{H}(t) =& \mathcal{H}_0\times_1 e^{(\hat{A}-I)t}\times_2 e^{(U-I)t} \times_3 e^{(W-I)t} \\
    &+ \int_{0}^{t}\mathcal{H}_0\times_1 e^{(\hat{A}-I)(t-\tau)}\times_2 e^{(U-I)(t-\tau)} \times_3 e^{(W-I)(t-\tau)}\mathrm{d}\tau
  \end{aligned}
  \end{equation}
\end{proof}
In fact, the last integration can be solved further, but limited by space, we put it in the supplementary.
Notice that the eigenvalues of $\hat{A}-I$ is in the interval $[-1, 0)$, as a result, $e^{(\hat{A}-I)t}$ will go to zero when t goes to $\infty$. However, unrestricted $U$ and $W$ will lead to divergent integrations as t goes to $\infty$. To enforce $U$ and $W$ to be a diagonalizable matrix with all the eigenvalues less than 1, we follow previous work \cite{cisse2017parseval} to parameterise $U$ and $W$ as $U=Pdiag(\lambda)P^T$ and $W=Qdiag(\mu)Q^T$ respectively, where $P$ and $Q$ are learnable orthogonal matrices, $\lambda$ and $\mu$ are learnable vectors whose elements will be clamped to the interval $(0,1)$.

So far, we have proved a continuous form of tensor-based hidden representation theoretically. Motivated by Neural ODE \cite{chen2018neural}, we propose an STGODE learning framework,
\begin{equation}\label{stgode}
  \mathcal{H}(t) = ODESolve\left(\frac{\mathrm{d} \mathcal{H}(t)}{\mathrm{d} t}, \mathcal{H}_0, t\right),
\end{equation}
where $$\frac{\mathrm{d}\mathcal{H}(t)}{\mathrm{d}t} = \mathcal{H}(t)\times_1 (\hat{A}-I) + \mathcal{H}(t)\times_2 (U-I) + \mathcal{H}(t)\times_3 (W-I) + \mathcal{H}_0,$$
$\mathcal{H}_0$ denotes the initial value, which comes from the upstream network and the ODESolver is chosen as the Euler solver in our model.

\subsection{Temporal Convolutional Blocks}
Besides spatial correlations among different nodes, the long-term temporal correlations of the nodes themselves also matter.
Although RNN-based models, like LSTM and GRU, are widely applied in time-series analysis, recurrent networks still suffer from some intrinsic drawbacks like time-consuming iterations, unstable gradients, and delayed responses to dynamic changes.

To enhance the performance of extracting long term temporal dependencies, a 1-D dilated temporal convolutional network along the time axis is adopted here.
\begin{equation}
  H_{tcn}^{l} = \left\{ \begin{aligned}
  &X&,\ & l=0 \\
  &\sigma(W^l *_{d^l} H_{tcn}^{l-1})&,\ & l=1,2,\cdots,L
  \end{aligned}
  \right.
\end{equation}
where $X\in \mathbb{R}^{N\times T\times F}$ is the input of TCN, $H_{tcn}^l\in \mathbb{R}^{N\times T\times F}$ is the output of the $l$-th layer of TCN,  and $W^{l}$ denotes the $l$-th convolution kernel. To expand the receptive field , an exponential dilation rate $d^l=2^{l-1}$ is adopted in temporal convolution. In the process, zero-padding strategy is utilized to keep time sequence length unchanged. What's more, a residual structure \cite{he2016deep} is added to strengthen convolution performance as shown in Fig \ref{fig:tcn}.

\subsection{STGODE Layer}
In this part, the overall STGODE layer is presented in detail. As illustrated in Fig \ref{fig:odelayer}, the "sandwich" structure is adopted which consist of two TCN blocks and a STGODE solver. Such structure enables flexible and sensible spatial-temporal information flows, and all-convolution structures have the superiority of fast training and parallelization. Stacked "sandwiches" further extend the model's ability to discover complex correlations.

In the construction of the model, we deploy two kinds of STGODE blocks, which accept different adjacency matrices, i.e. the spatial adjacency matrix and the semantic adjacency matrix. Two kinds of adjacency matrices are utilized to combine local dynamics and semantical relationships altogether, which greatly enhance the representation ability. Multiple blocks are deployed in parallel so that more complicated and multi-level correlations can be captured.

\subsection{Others}
After the STGODE layers, a max-pooling operation is carried out to aggregate information from different blocks selectively. Finally, a two-layer MLP is designed as the output layer to transform the output of the max-pooling layer to the final prediction.

Huber loss is selected as the loss function since it is less sensitive to outliers than the squared error loss \cite{huber1992robust},
\begin{equation}\label{eq:huber}
  L(Y, \hat{Y})=\left\{
  \begin{aligned}
  &\frac{1}{2}(Y-\hat{Y})^2 &,\ & |Y-\hat{Y}|\leq \delta \\
  &\delta |Y-\hat{Y}| - \frac{1}{2}\delta^2&,\ & \text{otherwise}\\
  \end{aligned}
  \right.
\end{equation}
where $\delta$ is a hyperparameter which controls the sensitivity to outliers.

\section{Experiments}
\subsection{Datasets}
We verify the performance of STGODE on six real-world traffic datasets, PeMSD7(M), PeMSD7(L), PeMS03, PeMS04, PeMS07, and PeMS08, which are collected by the Caltrans Performance Measurement System(PeMS) in real time every 30 seconds\cite{chen2001freeway}. The traffic data are aggregated into 5-minutes intervals, which means there are 288 time steps in the traffic flow for one day. The system has more than 39,000 detectors deployed on the highway in the major metropolitan areas in California. There are three kinds of traffic measurements contained in the raw data, including traffic flow, average speed, and average occupancy.

Specifically, PeMSD3 has 358 sensors, and the time span of it is from September to November in 2018, including 91 days in total. PeMSD7(M) and PeMSD7(L) are two datasets selected from District 7 of California, which contains 288 and 1,026 sensors respectively. The time range of PeMSD7 is in the weekdays of May and June of 2012. And PeMSD8 is collected from July to August in 2016, which contains 170 sensors. The detail of datasets is listed in Table \ref{tab:dataset}. Z-score normalization is applied to the input data, i.e. removing the mean and scaling to unit variance.
\begin{table}[htbp]
  \centering
    \begin{tabular}{cccc}
    \hline
    Datasets & \#Sensors & \#Edges & Time Steps  \\ \hline
    PeMSD7(M) & 228  &1132  & 12672 \\
    PeMSD7(L) & 1026 &10150 & 12672 \\
    PeMS03    & 358  &547   & 26208 \\
    PeMS04    & 307  &340   & 16992 \\
    PeMS07    & 883  &866   & 28224 \\
    PeMS08    & 170  &295   & 17856 \\ \hline
    \end{tabular}%
  \caption{Datasets description}
  \label{tab:dataset}%
\end{table}%

\subsection{Baselines}
We compare STODE with following baseline models:
\begin{itemize}
  \item \textbf{ARIMA} \cite{box1970distribution}: Auto-Regressive Integrated Moving Average model, which is a well-known statistical model of time series analysis.
  \item \textbf{STGCN} \cite{yu2018spatio}: Spatio-Temporal Graph Convolution Network, which utilizes graph convolution and 1D convolution to capture spatial dependencies and temporal correlations respectively.
  \item \textbf{DCRNN} \cite{li2018diffusion}: Diffusion Convolution Recurrent Neural Network, which integrates graph convolution into an encoder-decoder gated recurrent unit.
  \item \textbf{GraphWaveNet} \cite{wu2019graph}: Graph WaveNet, which combines adaptive graph convolution with dilated casual convolution to capture spatial-temporal dependencies.
  \item \textbf{ASTGCN(r)} \cite{guo2019attention}: Attention based Spatial Temporal Graph Convolutional Networks, which utilize spatial and temporal attention mechanisms to model spatial-temporal dynamics respectively. In order to keep the fairness of comparison, only recent components of modeling periodicity are taken.
  \item \textbf{STSGCN} \cite{song2020spatial}: Spatial-Temporal Graph Synchronous Graph Convolutional Networks, which utilize multiple localized spatial-temporal subgraph modules to synchronously capture the localized spatial-temporal correlations directly.
  \end{itemize}

\subsection{Experimental Settings}
We split all datasets with a ratio 6: 2: 2 into training sets, validation sets, and test sets. One hour of historical data is used to predict traffic conditions in the next 60 minutes.

All experiments are conducted on a Linux server(CPU: Intel(R) Xeon(R) CPU E5-2682 v4 @ 2.50GHz, GPU: NVIDIA TESLA V100 16GB). The hidden dimensions of TCN blocks are set to 64, 32, 64, and 3 STGODE blocks are contained in each layer. The regularized hyperparameter $\alpha$ is set to 0.8, the thresholds $\sigma$ and $\epsilon$ of the spatial adjacency matrix are set to 10 and 0.5 respectively, and the threshold $\epsilon$ of the semantic adjacency matrix is set to 0.6.

We train our model using Adam optimizer with a learning rate of 0.01. The batch size is 32 and the training epoch is 200. Three kinds of evaluation metrics are adopted, including root mean squared errors(RMSE), mean absolute errors(MAE), and mean absolute percentage errors(MAPE).

\begin{table*}[htbp]
  \centering
    \begin{tabular}{ccccccccc}
    \toprule
    Dataset & Metric & ARIMA & STGCN & DCRNN & ASTGCN(r) & GraphWaveNet & STSGCN & \textbf{STODE} \\
    \midrule
          & RMSE  & 13.20 & 7.55  & 7.18  & 6.87  & 6.24  & 5.93  & \textbf{5.66} \\
    PeMSD7(M) & MAE   & 7.27  & 4.01  & 3.83  & 3.61  & 3.19  & 3.01  & \textbf{2.97} \\
          & MAPE  & 10.38 & 9.67  & 9.81  & 8.84  & 8.02  & 7.55  & \textbf{7.36} \\
    \midrule
          & RMSE  & 12.39 & 8.28  & 8.33  & 7.64  & 7.09  & 6.88  & \textbf{5.98} \\
    PeMSD7(L) & MAE   & 7.51  & 4.84  & 4.33  & 4.09  & 3.75  & 3.61  & \textbf{3.22} \\
          & MAPE  & 15.83 & 11.76 & 11.41 & 10.25 & 9.41  & 9.13  & \textbf{7.94} \\
    \midrule
          & RMSE  & 47.59 & 30.42 & 30.31 & 29.56 & 32.77 & 29.21 & \textbf{27.84} \\
    PeMS03 & MAE   & 35.41 & 17.55 & 17.99 & 17.34 & 19.12 & 17.48 & \textbf{16.50} \\
          & MAPE  & 33.78 & 17.43 & 18.34 & 17.21 & 18.89 & 16.78 & \textbf{16.69} \\
    \midrule
          & RMSE  & 48.80 & 36.01 & 37.65 & 35.22 & 39.66 & 33.65 & \textbf{32.82} \\
    PeMS04 & MAE   & 33.73 & 22.66 & 24.63 & 22.94 & 24.89 & 21.19 & \textbf{20.84} \\
          & MAPE  & 24.18 & 14.34 & 17.01 & 16.43 & 17.29 & 13.90 & \textbf{13.77} \\
    \midrule
          & RMSE  & 59.27 & 39.34 & 38.61 & 37.87 & 41.50 & 39.03 & \textbf{37.54} \\
    PeMS07 & MAE   & 38.17 & 25.33 & 25.22 & 24.01 & 26.39 & 24.26 & \textbf{22.99} \\
          & MAPE  & 19.46 & 11.21 & 11.82 & 10.73 & 11.97 & 10.21 & \textbf{10.14} \\
    \midrule
          & RMSE  & 44.32 & 27.88 & 27.83 & 26.22 & 30.04 & 26.80 & \textbf{25.97} \\
    PeMS08 & MAE   & 31.09 & 18.11 & 17.46 & 16.64 & 18.28 & 17.13 & \textbf{16.81} \\
          & MAPE  & 22.73 & 11.34 & 11.39 & 10.6  & 12.15 & 10.96 &  \textbf{10.62} \\
    \bottomrule
    \end{tabular}%
  \caption{Performance comparison of baseline models and STGODE on PeMS datasets.}
\label{tab:result}
\end{table*}%

\subsection{Experimental Results and Analysis}


Table \ref{tab:result} shows the results of our and competitive models for traffic flow forecasting. Our STGODE model is obviously superior to the baselines.
Specifically, deep learning methods achieve better results than traditional statistical methods, as traditional methods like ARIMA only take temporal correlations into consideration and ignore spatial dependencies, whereas deep learning models can take advantage of spatial-temporal information. Among the deep learning baselines, all except STSGCN utilize two modules to model spatial dependencies and temporal correlations respectively, which overlook complex interactions between spatial information and temporal information, and STSGCN hence surpasses other models. But STSGCN only concentrates on localized spatial-temporal correlations, and turns turtle in global dependencies.

Our model yields the best performance regarding all the metrics for all datasets, which suggests the effectiveness of our spatial-temporal dependency modeling. The result can be attributed to three aspects:
\begin{enumerate}
  \item We utilize a tensor-based ODE framework to extract longer-range spatial-temporal dependencies;
  \item The semantical neighbors are introduced to establish global and comprehensive spatial relationships;
  \item Temporal dilated convolution networks with residual connections help to capture long term temporal dependencies.
\end{enumerate}

\subsection{Case Study}
Here we select two nodes from the road network to carry out a case study. As Fig \ref{fig:case} shows, the prediction results of STGODE are remarkably closer to the ground truth than STGCN \cite{yu2018spatio}. In normal circumstances, the model generates a smooth prediction ignoring small oscillations to fight against noise. But when an abrupt change arises, our model enables a rapid response to it. This is because STGODE is able to utilize feature information from longer range geographical neighbors and semantic neighbors, which helps to accurately capture real-time dynamics and filter invalid information, while STGCN as a shallow network, is susceptible to few nearby neighbors and thus performs unstably. 
\begin{figure}[ht]
  \centering
  \subfigure[]{
  \includegraphics[width=0.45\linewidth]{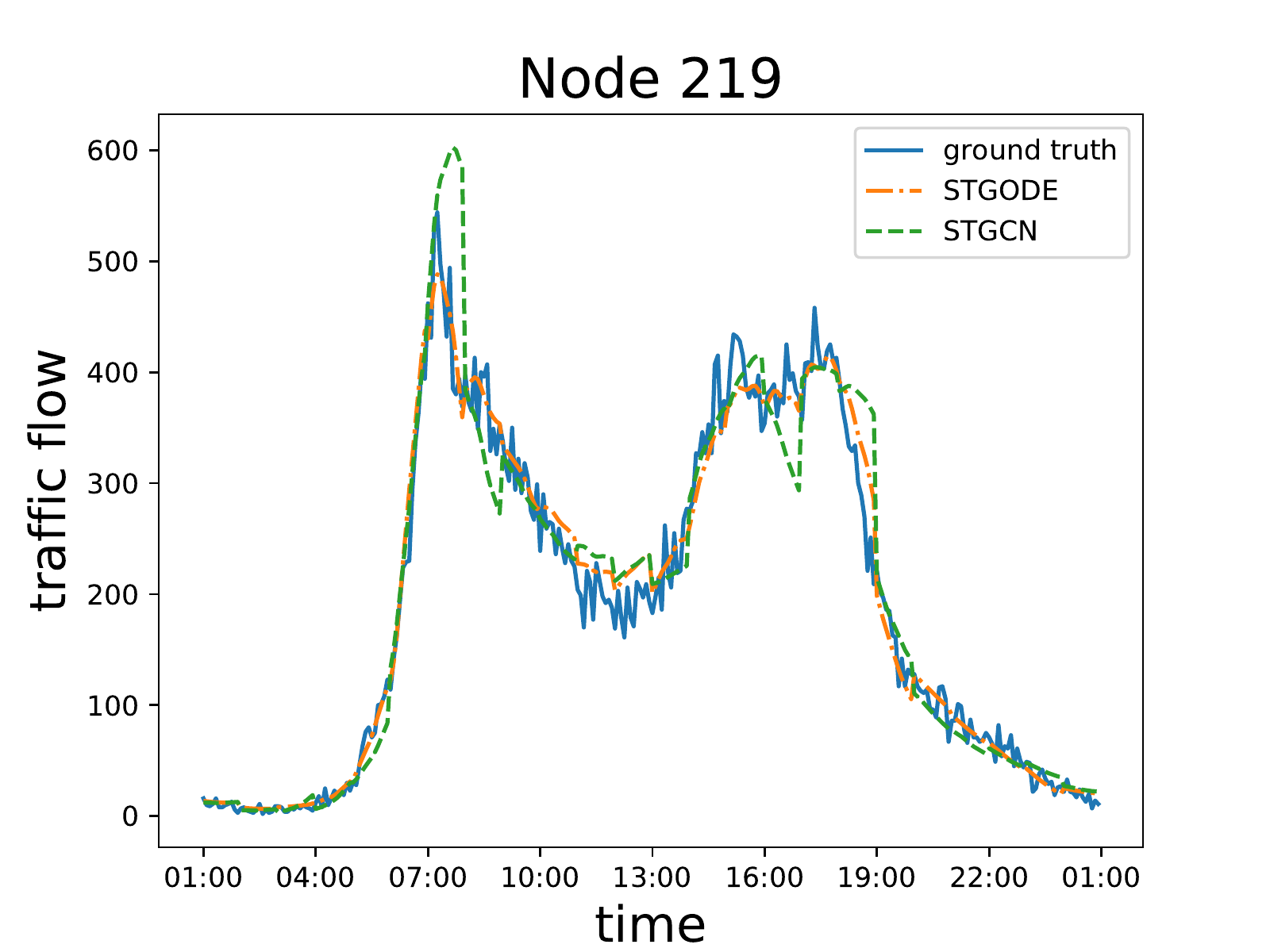}}
  \subfigure[]{
  \includegraphics[width=0.45\linewidth]{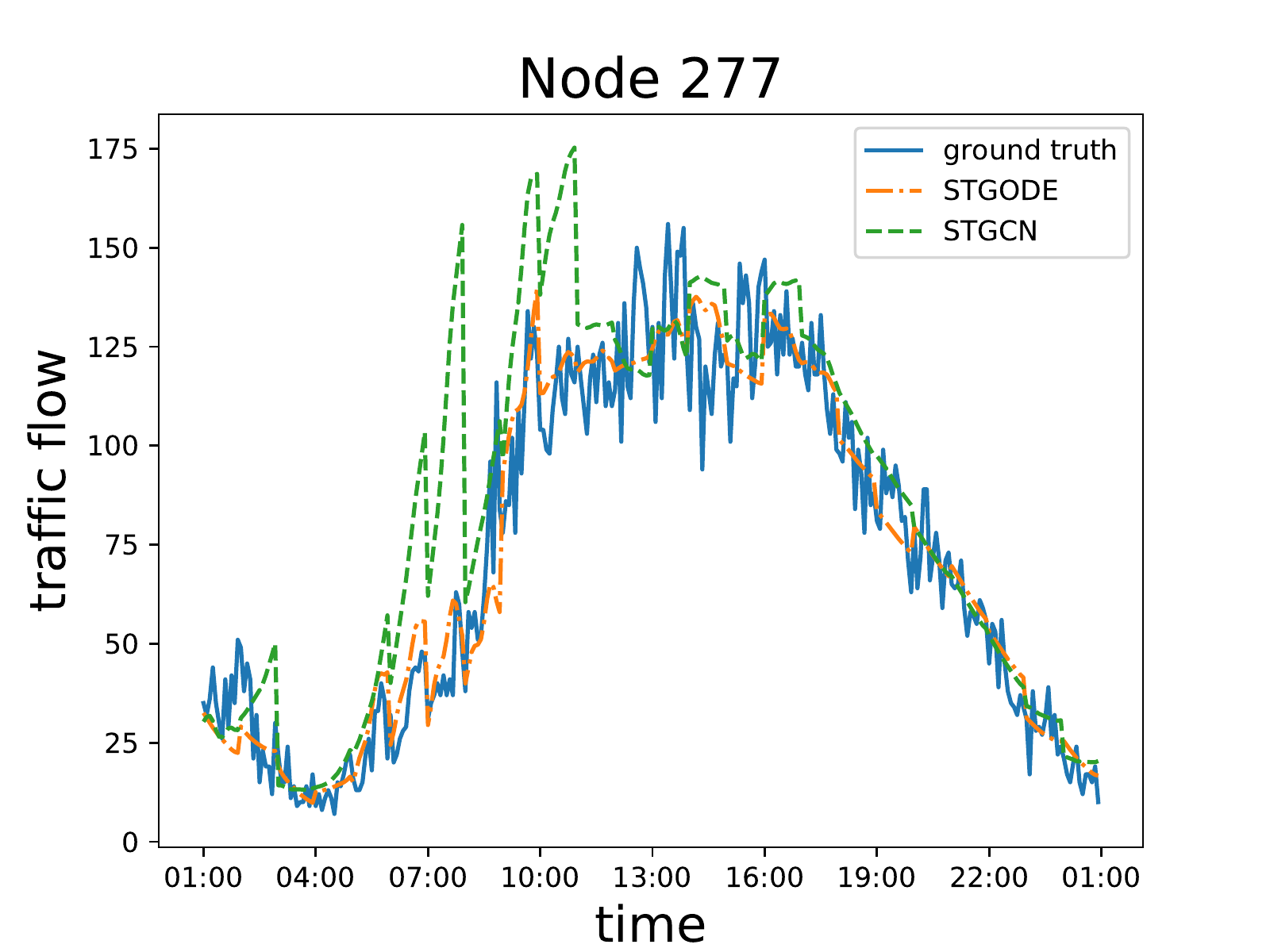}}
  \caption{The comparison of prediction results between our model and STGCN.}
  \label{fig:case}
\end{figure}

\subsection{Model Analysis}
\subsubsection{Ablation Experiments}
To verify the effectiveness of different modules of STGODE, we conduct the following ablation experiments on PeMS04 dataset, and four variants of STGODE are designed.
\begin{itemize}
  \item STGCN*: The ODE solver is replaced with a graph convolution layer to verify the effectiveness of ODE structures for extracting long-range dependencies.
  \item STGODE only spatial: This model does not consider semantic neighbors to verify the necessity of introducing a semantic adjacency matrix.
  \item STGODE-no-h0: The initial state is removed in the derivation of hidden states (Eq \ref{eq:ode}).
  \item STGODE-matrix-based: Reformulate the tensor-based ODE (Eq \ref{eq:ode}) to a matrix-based version as following,
  \begin{equation}\label{matrix-based}
    \frac{\mathrm{d}H(t)}{\mathrm{d}t} = \ln \hat{A} H(t) + H(t) \ln W + H_0
  \end{equation}
   which means that the input tensor will be viewed as multiple matrices separately without considering temporal feature transform in ODE blocks.
\end{itemize}

\begin{figure}[ht]
  \centering
  \includegraphics[width=0.99\linewidth]{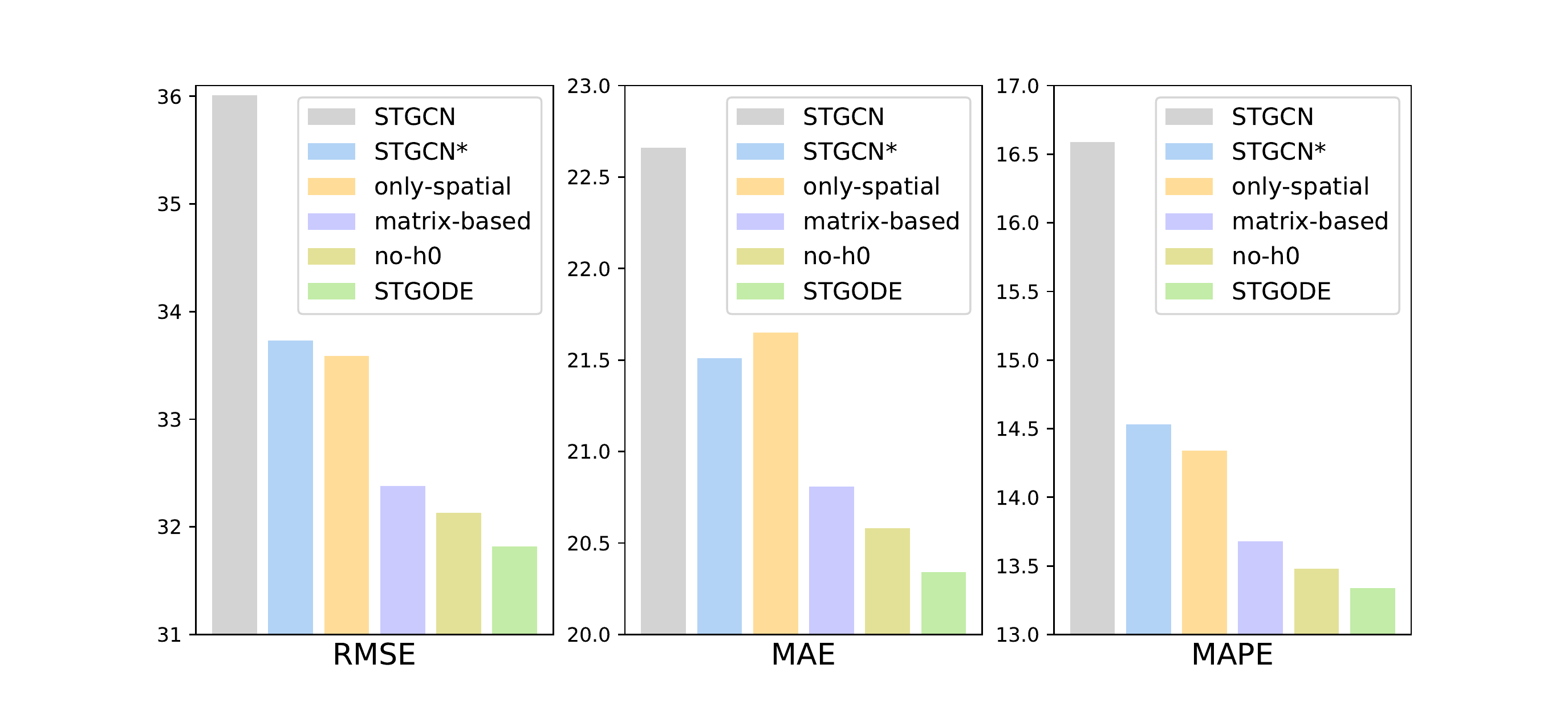}
  \caption{Ablation experiments of STGODE}\label{fig:ablation}
\end{figure}
The results are presented in Fig \ref{fig:ablation}. Here we put STGCN and our STGCN* together on account of their similar sandwich structures and the same way of convolution. The result shows that our STGCN* performs much better than previous STGCN, which is contributed to our novel temporal convolution and the introduction of semantical neighbors, and the poor result of STGODE with only spatial neighbors reinforces the latter point. The performance of the matrix-based version is also inferior to the tensor-based one, as it is incapable to consider spatial-temporal dependency simultaneously. And the result of STGODE without $\mathcal{H}_0$ shows the importance of connecting the initial state.

\subsubsection{Parameter Analysis}
\begin{figure}
  \centering
  \subfigure[]{
  \includegraphics[width=0.45\linewidth]{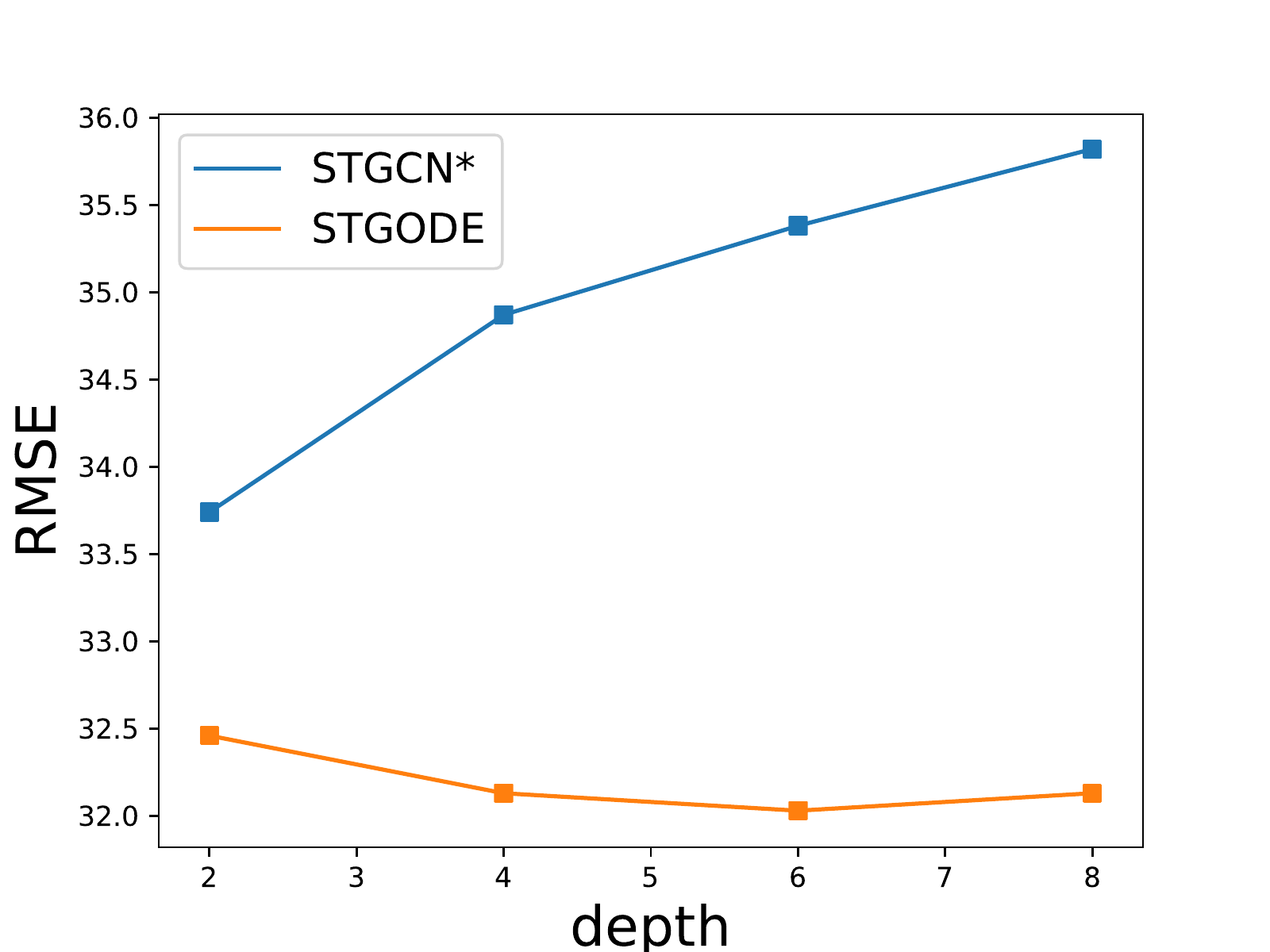}}
  \subfigure[]{
  \includegraphics[width=0.45\linewidth]{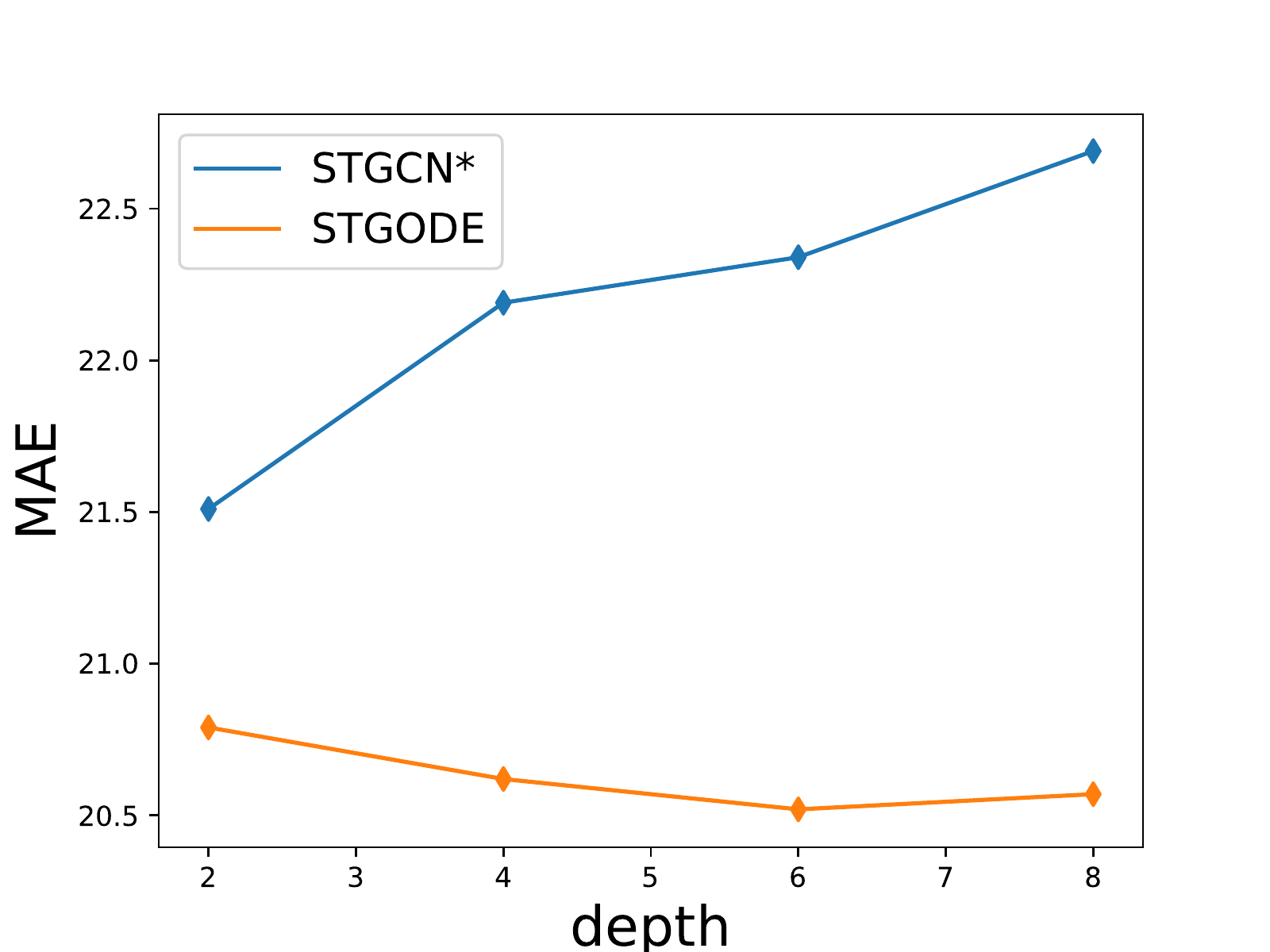}}
  \caption{The performance of STGODE and STGCN when the network depth increasing.}\label{fig:analysis_t}
\end{figure}
One major advantage of our STGODE model over other existing methods is that is robust to the over-smoothing problem and thus capable to construct deeper network structures. Here in Fig \ref{fig:analysis_t}, we represent the performance of STGODE and STGCN* under different depths, i.e. the input time length of STGODE solver and the number of convolution layers in STGCN*. It is easy to see that, as the network depth increases, the performance of STGCN* drops dramatically while the performance of our model is stable, which clearly shows the strong robustness of our model to extract longer-range dependencies.

\section{Conclusion}
A tremendous number of works have been proposed to tackle the complex spatial-temporal problems, but few of them focus on how to extract long-range dependencies without being affected by the over-smoothing problem. In this paper, we present a novel tensor-based spatial-temporal forecasting model named STGODE. To the best of our knowledge, this is the first attempt to bridge continuous differential equations to the node representations of road networks in the area of traffic, which enables to construct deeper networks and leverage wider-range dependencies. Furthermore, the participation of semantic neighbors largely enhances the performance of the model. Extensive experiments prove the effectiveness of STGODE over many existing methods.

\begin{acks}
This work was supported by the National Natural Science Foundation of China (Grant No. 61876006 and No. 61572041).
\end{acks}

\bibliographystyle{ACM-Reference-Format}
\bibliography{main}

\newpage
\section{Appendix}
\subsection{The calculation of the integration in Eq 19}
\begin{proof}
Suppose $\hat{A}-I,U-I,W-I$ have eigenvalue decompositions $P_1\Lambda_1P_1^{-1}, P_2\Lambda_2P_2^{-1}, P_3\Lambda_3P_3^{-1}$ respectively, then we have\begin{align*}
  &\int_{0}^{t}\mathcal{H}_0\times_1 e^{(\hat{A}-I)(t-\tau)}\times_2 e^{(U-I)(t-\tau)} \times_3 e^{(W-I)(t-\tau)}\mathrm{d}\tau \\
  =& \int_{0}^{t}\mathcal{H}_0\times_1 P_1e^{\Lambda_1(t-\tau)} P_1^{-1} \times_2 P_2e^{\Lambda_2(t-\tau)} P_2^{-1} \times_3 P_3e^{\Lambda_3(t-\tau)} P_3^{-1}\mathrm{d}\tau \\
  =& \int_{0}^{t}\mathcal{H}_0\times_1 P_1 \times_2 P_2 \times_3 P_3 \times_1 e^{\Lambda_1(t-\tau)} \times_2 e^{\Lambda_2(t-\tau)}\\ &\times_3 P_3e^{\Lambda_3(t-\tau)} \times_1 P_1^{-1} \times_2 P_2^{-1}\times_3 P_3^{-1} \mathrm{d}\tau,
\end{align*}

denote $\tilde{\mathcal{H}_0} = \mathcal{H}_0\times_1 P_1 \times_2 P_2 \times_3 P_3$, 
\begin{align*}
  &\int_{0}^{t}\mathcal{H}_0\times_1 e^{(\hat{A}-I)(t-\tau)}\times_2 e^{(U-I)(t-\tau)} \times_3 e^{(W-I)(t-\tau)}\mathrm{d}\tau \\
  =& \int_{0}^{t} \tilde{\mathcal{H}_0} \times_1 e^{\Lambda_1(t-\tau)} \times_2 e^{\Lambda_2(t-\tau)} \times_3 e^{\Lambda_3(t-\tau)}\mathrm{d}\tau \times_1 P_1^{-1} \times_2 P_2^{-1}\times_3 P_3^{-1},  \\
\end{align*}
consider the integral element-wise, then we have,  
\begin{align*}
  &\left( \int_{0}^{t} \tilde{\mathcal{H}_0} \times_1 e^{\Lambda_1(t-\tau)} \times_2 e^{\Lambda_2(t-\tau)} \times_3 e^{\Lambda_3(t-\tau)}\mathrm{d}\tau \right)_{ijk}\\
  =& \int_{0}^{t} \tilde{\mathcal{H}}_{0ijk} \times_1 e^{\Lambda_{1ii}(t-\tau)} \times_2 e^{\Lambda_{2jj}(t-\tau)} \times_3 e^{\Lambda_{3kk}(t-\tau)}\mathrm{d}\tau \\
  =& \left. -\frac{1}{\Lambda_{1ii}+\Lambda_{2jj}+\Lambda_{3kk}}\tilde{\mathcal{H}}_{0ijk} \times_1 e^{\Lambda_{1ii}(t-\tau)} \times_2 e^{\Lambda_{2jj}(t-\tau)} \times_3 e^{\Lambda_{3kk}(t-\tau)} \right|_0^t \\
  =& \frac{\tilde{\mathcal{H}}_{0ijk}}{\Lambda_{1ii}+\Lambda_{2jj}+\Lambda_{3kk}} \times_1 e^{\Lambda_{1ii}t} \times_2 e^{\Lambda_{2jj}t} \times_3 e^{\Lambda_{3kk}t} - \frac{\tilde{\mathcal{H}}_{0ijk}}{\Lambda_{1ii}+\Lambda_{2jj}+\Lambda_{3kk}}
\end{align*}
thus, the result of the integration is as the following,
\begin{align*}
  &\int_{0}^{t}\mathcal{H}_0\times_1 e^{(\hat{A}-I)(t-\tau)}\times_2 e^{(U-I)(t-\tau)} \times_3 e^{(W-I)(t-\tau)}\mathrm{d}\tau \\ 
  =& \left(\frac{\tilde{\mathcal{H}}_{0ijk}}{\Lambda_{1ii}+\Lambda_{2jj}+\Lambda_{3kk}} \times_1 e^{\Lambda_{1ii}t} \times_2 e^{\Lambda_{2jj}t} \times_3 e^{\Lambda_{3kk}t} - \frac{\tilde{\mathcal{H}}_{0ijk}}{\Lambda_{1ii}+\Lambda_{2jj}+\Lambda_{3kk}} \right) \\
  &\times_1 P_1^{-1} \times_2 P_2^{-1}\times_3 P_3^{-1}.
\end{align*}

\end{proof}

\end{document}